\documentclass[runningheads]{llncs}
\usepackage[margin=1in]{geometry}
\usepackage{graphicx}
\usepackage{comment}
\usepackage[subtle]{savetrees}

\usepackage{amsmath}

\DeclareMathOperator*{\argmin}{arg\,min}

\usepackage{amssymb}
\usepackage{color}
\usepackage{soul}
\usepackage{hyperref}
\newcommand{\rev}[1]{#1}

\begin{document}
\title{On minimizers and convolutional filters: theoretical connections and applications to genome analysis\thanks{Supported by Natural Sciences and Engineering
Research Council of Canada (NSERC) grant RGPIN-2022-03074.}}

\titlerunning{On minimizers and convolutional filters}
%
\author{Yun William Yu\inst{1, 2}\orcidID{0000-0002-8275-9576}}
\authorrunning{Y. W. Yu}
%
\institute{Department of Mathematics, University of Toronto, Toronto, Ontario, M5S 1A1, Canada \and
Ray and Stephanie Lane Computational Biology Department, Carnegie Mellon University, Pittsburgh, PA 15213, USA \\
\email{ywyu@cmu.edu}}
\maketitle              
\begin{abstract}
Minimizers and convolutional neural networks (CNNs) are two quite distinct popular techniques that have both been employed to analyze categorical biological sequences. At face value, the methods seem entirely dissimilar. Minimizers use min-wise hashing on a rolling window to extract a single important k-mer feature per window. CNNs start with a wide array of randomly initialized convolutional filters, paired with a pooling operation, and then multiple additional neural layers to learn both the filters themselves and how they can be used to classify the sequence.

Here, our main result is a careful mathematical analysis of hash function properties showing that for sequences over a categorical alphabet, random Gaussian initialization of convolutional filters with max-pooling is equivalent to choosing a minimizer ordering such that selected k-mers are (in Hamming distance) far from the k-mers within the sequence but close to other minimizers.
In empirical experiments, we find that this property manifests as decreased density in repetitive regions, both in simulation and on real human telomeres.
We additionally train from scratch a CNN embedding of synthetic short-reads from the SARS-CoV-2 genome into 3D Euclidean space that locally recapitulates the linear sequence distance of the read origins, a modest step towards building a deep learning assembler, though it is at present too slow to be practical.
In total, this manuscript provides a partial explanation for the effectiveness of CNNs in categorical sequence analysis.


\end{abstract}

\newpage

\section{Introduction}

It is a pithy statement now that the near exponential explosion of biological sequence data we confront requires the construction of more efficient tailored algorithms capable of handling the data deluge \cite{kahn2011future,berger2016computational}.
Over the last couple decades, local k-mer subsampling schemes such as minimizers \cite{schleimer2003winnowing,roberts2004reducing}, syncmers \cite{edgar2021syncmers}, and minimally-overlapping words \cite{frith2020minimally} have been applied to reduce the redundancy found in the overlapping neighboring k-mers of a sequence, and thus allow speeding up tasks like taxonomic classification \cite{wood2019improved}, Average Nucleotide Identity estimation \cite{shaw2023fast}, read mapping \cite{li2016minimap}, or assembly \cite{ekim2021minimizer}.
These methods have become necessary because although Moore's law on transistor density has continued unabated, the same cannot be said about single-threaded processing times or fast memory access.

However, it bears remarking that even as traditional genomics analysis has turned to increasingly efficient algorithmics, the burgeoning subarea of deep learning in biology has exploded \cite{angermueller2016deep,fiannaca2018deep}, and the neural networks being used for the same tasks (such as metagenomic classification) are by comparison much less computationally efficient.
Instead, deep learning is able to continue to tap into Moore's law because its computational primitive of numerical linear algebra is easily vectorized onto hardware accelerators such as GPUs (graphics processing units) and TPUs (tensor processing units) \cite{lu2010k3,jouppi2017datacenter}.
On bioinformatics tasks that map directly onto traditional ML tasks---e.g. classification and prediction of images---those models often generally outperform traditional methods on accuracy \cite{deepak2019brain,liu2018enhancing}.
Furthermore, deep learning methods are often able to achieve comparable accuracy even on core computational biology tasks like variant calling \cite{poplin2018universal} and metagenomic binning \cite{yoshimura2021genomic}, but at heavy computational cost.
Thus, choosing the right analysis methodology requires carefully navigating the relevant trade-offs \cite{berger2022navigating}.

However, although in theory with unbounded compute, neural networks can learn any function \cite{hornik1989multilayer}, a lot of work has been done to create more efficient architectures, including convolutional neural nets (CNNs) \cite{fukushima1982neocognitron}, recurrent neural nets (RNNs) \cite{rumelhart1986learning}, long-short-term memory structures (LSTM) \cite{hochreiter1997long}, Transformer networks \cite{vaswani2017attention}, and more.
However, although some ML models have an intuitive justification, often, they are simply found to perform well empirically, and it is difficult to interpret why they work, or what the learned weights might mean biologically \cite{montavon2018methods}.

\begin{figure}[t]
    \centering
	\includegraphics[width=0.49\columnwidth]{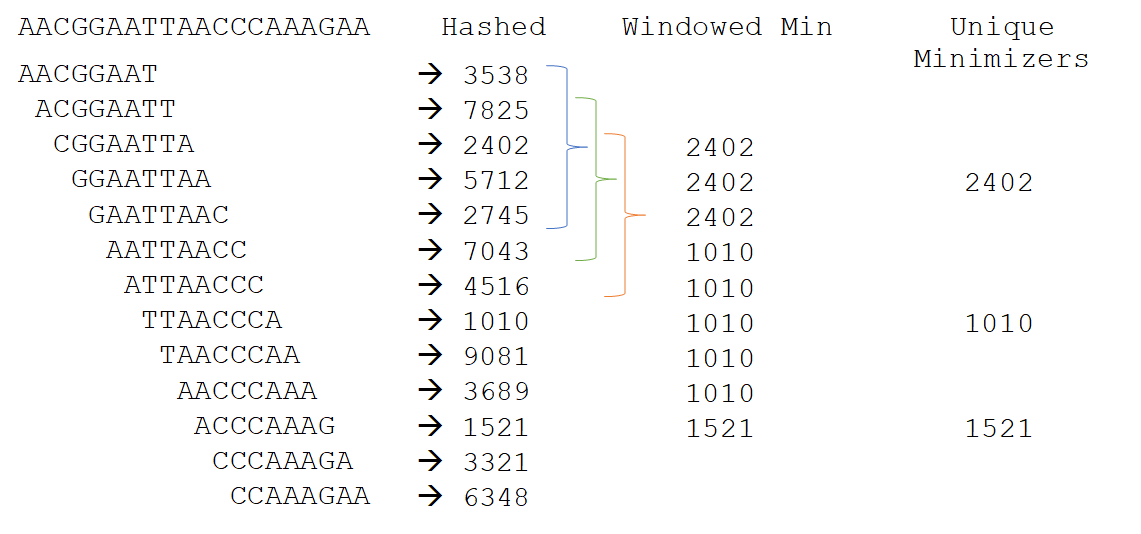}
	\vrule
	\hfill
	\includegraphics[width=0.49\columnwidth]{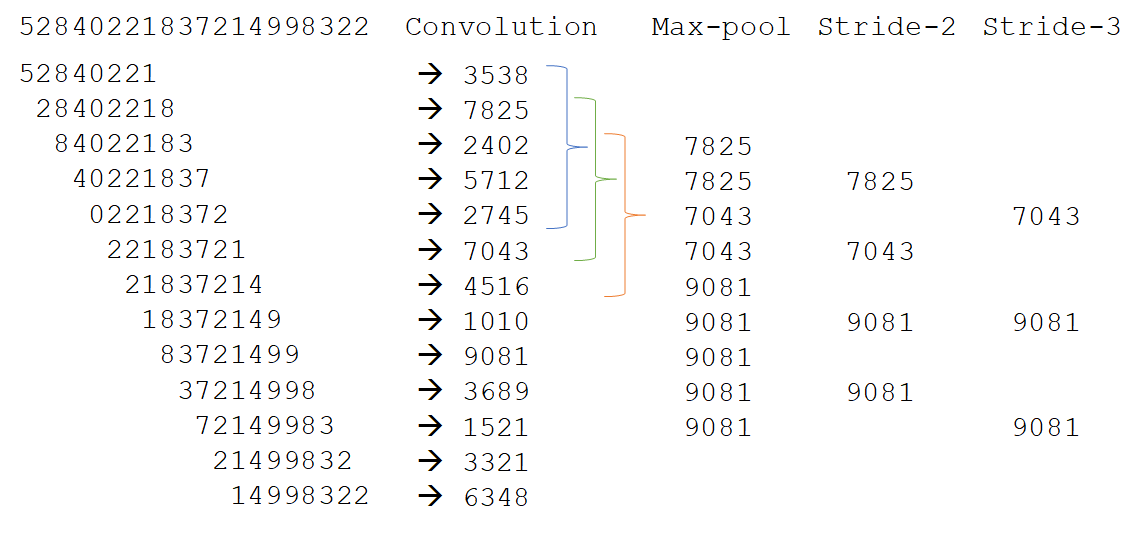}
		\vspace{-1.5em}
		\caption{\textbf{(left)} \textit{Construction of minimizers in bioinformatics algorithms}. A string of nucleotides is separated into overlapping k-mers ($k=8$ in figure). Those k-mers are hashed---typically randomly hashed---to integers, and the minimum integer within each overlapping window (length 5 in figure) is computed, called the `minimizer'. Only unique minimizers are stored, giving a sparse representation. 		
		\textbf{(right)} \textit{Applying a 1D convolutional filter with max-pooling}. A vector of values is convolved with a weights vector (here of length 8) to determine the output of a single filter. A max-pooling operation takes the maximums in some window (window size here is 5), and sparsification can be achieved by increasing stride-length. These maximums are features for the remaining layers of a CNN.}\label{fig:cnn}		
    \label{fig:my_label}
    \vspace{-1.5em}
\end{figure}

In this manuscript, we focus on understanding the efficacy of CNNs on categorical sequence data.
While interpretations of a CNNs weights and neurons are often made by way of comparison to the receptive fields of visual neurons, that interpretation falls apart when applied to biological sequence data. Unlike the brightness of a pixel in an image, biological sequences are not ordinal, but rather instead categorical vectors---a nucleotide is either present or not in a particular position, without any sense of relative magnitude.
In images, for example, a filter may detect a horizontal edge, which is characterized by bright pixels on one side and dark on the other; if you slightly brighten the pixels on one side, or add some noise to the image, the edge is still characterized by that pattern, which filters are trained to pick up.
However, while you could imagine patterns of nucleotides that correspond to an edge, e.g. AAAATTTT, there are no intermediate `levels' for a CNN filter to pick up.
Still, CNNs have proved effective on genomic data, such as in the classification of metagenomic sequences of bacteria by taxonomy \cite{fiannaca2018deep}.
It thus seems that CNNs are not intuitively suited for categorical sequence classification, but empirically, they irrespectively have excellent performance. Why?

Here, we prove two technical Theorems \ref{thm:math} and \ref{thm:math2} showing that the hash family of dot products with random spherical Gaussian multivariates preferentially selects more extremal elements as maximums, but uniformly randomly selects among equally distinct elements, and that maximums are more likely to be similar to each other.
A immediate corollary is then our Main Theorem \ref{thm:main}, showing that on categorical sequences, minimizers and untrained randomly initialized convolutional filters with max-pooling (Figure \ref{fig:cnn}) are nearly equivalent.


We then show in empirical simulations that these properties manifest as decreased \textit{density} in highly repetitive sequences as compared to random minimizer schemes, both in simulation and on real human telomeres.
Although at least one minimizer is guaranteed to be selected from each window, the number of minimizers that are selected overall determines the density; lowering the density can improve the efficiency of minimizer-based algorithms, as has been a design goal in the minimizer literature \cite{marccais2017improving,zheng2020improved}.
Our results show that Gaussian convolutions are better for density than random minimizers, though special-purpose methods such as Miniception\cite{zheng2020improved} still do better than Gaussian convolutions.

Lastly, we demonstrate application of these ideas to genome assembly.
To our knowledge, there exist no deep learning assemblers---a few studies apply deep learning to consensus-sequence determination \cite{vrvcek2022learning,padovani2019machine}, but nothing for generation of the assembly graph itself---which would seem to suggest that these problems are not amenable to ML.
We take inspiration from the convolutional filter/minimizer equivalence to train a neural embedding of synthetic short-reads from the SARS-COV-2 genome into 3D Euclidean space that respects linear distance along the genome of the read origins.
Although this is not a full assembler, it does show that at least in principle, the main task of assembly is amenable to deep learning, even in a fully de novo manner, though as might be expected, the computational cost of doing so is likely currently prohibitive.
Together, our theory and experiments resolve connect together two seemingly unrelated methods that have both achieved powerful results in computational biology, explaining some of the advantages of using random Gaussian convolutions for feature selection, and hinting at novel ways of combining deep learning with more traditional computational biology algorithms.

\section{Building blocks}


\subsection{Hashing and permutations}
The key to our analysis will be in understanding the role of hashing \cite{carter1979universal} and permutations. We recall some basic concepts here.
For reference, see any standard text
or review \cite{thorup2015high}.

\begin{definition}
A hash function $h:U\to M$ maps elements from some universe $U$ to some range $M$. 
\end{definition}
\begin{definition}
A random hash function $h:U\to M$ is a function that is uniformly drawn from a family $\mathcal{H}$ of hash functions. 
\end{definition}
To illustrate, let's start with a classical example.
Let $\mathbb{F}_p$ be the finite field with $p$ elements, for some prime $p$.
Then we can define a hash function $h_{a,b}: \mathbb{F}_p \to \mathbb{F}_p$ given by $h_{a,b}(x) = ax + b \mod p$.
The hash family here is parameterized by $a, b\in \mathbb{F}_p$, and so our analysis can operate on probabilities and expectations that arise from treating $a, b$, and therefore $h_{a,b}$, as random variables.

Another more sophisticated modern example is the vector multiply-shift family of hash functions, introduced by Dietzfelbinger, designed to hash vectors to scalars without making use of finite field arithmetic \cite{dietzfelbinger2018universal}.
We introduce it here because the construction is mathematically equivalent to a dot product with uniform random integers in the integer ring, followed by an integer division/bit-shift operation.
Let $U = [2^w]^d$ and $M = [2^l]$, where $[o]$ is the integer ring $\{0, 1, \ldots, o-1\}$.
We will need $\bar{w}\ge w+l-1$, and then we pick a uniform random vector $\mathbf{a} \in [2^{\bar{w}}]^d$, as well as a uniform random element $b \in [2^{\bar{w}}]$.
We define a hash function $h_{\mathbf{\mathbf{a}, b}}: U \to M$
given by $h_{\mathbf{\mathbf{a}, b}} (\mathbf{x}) = (\mathbf{x} \cdot \mathbf{a} + b) \operatorname{div} 2^{\bar{w}-l}$,
where the dot product and addition are in the ring $[2^{\bar{w}}]$,
and $\operatorname{div}$ is ordinary integer division without remainder (i.e. a bit-shift to the right by $\bar{w}-l$ bits), leaving an answer that is just $l$ bits.
In practice, using fast 64-bit integer arithmetic, we can let $l=32$, $w=32$, and $\bar{w}=64$, allowing us to hash vectors of unsigned 32-bit integers by doing a dot product with a random 64-bit integer vector, adding another random 64-bit integer, and then taking the higher 32-bits.

Depending on the application at hand, different properties of a hash function may be desirable, such as universality \cite{carter1979universal}, strong universality, k-independence \cite{wegman1979new}, or min-wise independence \cite{broder2000min}.
Both of the examples given above exhibit strong universality (also known as 2-independence).
\begin{definition}
$\mathcal{H}$ is a strongly universal (or 2-wise independent) hash family if $\forall i_1\ne i_2 \in U$, and $\forall j_1, j_2 \in M$,
\begin{equation}
    \Pr_{h \sim H} \left( h(i_1)=j_1 \wedge h(i_2)=j_2 \right) = \frac{1}{|M|^2}
\end{equation}
\end{definition}
Roughly speaking, any two hash values can be construed as independent uniform random variables on the range.

However, in this work, we are primarily interested in permutations.
Nontrivial permutations are expensive to encode, but they can be approximated via random hashing.
This is the basis of the celebrated MinHash algorithm for computing set similarity \cite{broder1997resemblance}.
Note that for any set $S \subset U$ with no collisions under a hash function $h$---i.e. $h(s_i)\ne h(s_j)$ for any $s_i \ne s_j \in S$---the hash function defines an ordering/permutation on $S$.
One property of random permutations people have sought to capture is that every item in $S$ has an equal chance of being the smallest, which we can formalize.
\begin{definition}
A family of hash functions $\mathcal{H}$ is min-wise independent if for any set $S\subseteq U$, and any $s \in S$,\label{def:minwise}
\begin{equation}
    \Pr_{h \sim H} \left( \min \{ h(S) \} = h(s) \right) = \frac{1}{|S|}
\end{equation}
\end{definition}
Min-wise independence unfortunately does not follow as a consequence of strong universality, though you can approximate it with sufficiently high degrees of k-independence \cite{indyk2001small} or by using a twisted variant of tabulation hashing \cite{dahlgaard2014approximately}.

Alternately, many real-world implementations eschew the formal guarantees of random hashing, and instead just use a deterministic hash function---such as Murmurhash3 \cite{appleby2008murmurhash}, the SHA family of cryptographic hash functions \cite{eastlake2001us}, or efficient canonical choices of group generators in prime fields.
This is very common in bioinformatics software making use of minimizers. However, such constructions obscure the connections we will be drawing in this paper.

\subsection{Minimizers}
\label{ss:minimizers}
Minimizers \cite{schleimer2003winnowing,roberts2004reducing} are a local k-mer selection scheme, in some ways, the classic k-mer selection scheme.
The most important feature of a local k-mer selection scheme is translation invariance; we want to subsample the set of k-mers in a sequence in such a way that even if we insert or delete a letter at the beginning of the sequence, the set of minimizers does not change very much.
There are a number of more modern k-mer selection schemes \cite{edgar2021syncmers,frith2020minimally,shaw2021theory}, with slightly different properties, but minimizers were among the first used in computational biology.

Minimizers are related to min-wise hashing \cite{broder1997resemblance}, but instead of getting features for an entire sequence at once, they break up a sequence into smaller windows and get the minimum hash within each window.
That minimum k-mer, a \textit{minimizer}, is used to match the sequence to some reference, e.g. for sequence assembly \cite{ekim2021minimizer}, comparison \cite{shaw2023fast}, or mapping \cite{li2016minimap}. Often, for longer sequences, we match the sequence to a reference only if multiple minimizers from different windows match.
The compressive nature of minimizers appears because most of the time, the minimizer remains constant as the window rolls, so the total number of minimizers for a sequence is much smaller than the total number of windows.
Indeed, one of the key metrics considered for minimizer schemes is the \textit{density}, defined as the fraction of all k-mers in a sequence that are selected.


Let's more precisely state a few standard results. Consider an alphabet $\Sigma$. For genomics, $\Sigma = \{A, C, G, T\}$, whereas for protein sequences $|\Sigma|=20$ in the standard amino acid alphabet.
We are interested in analyzing the set of variable-length strings $\Sigma^* = \Sigma^1 \cup \Sigma^2 \cup \Sigma^3 \cup \cdots$.
Given a length-$l$ string $x = x_0\cdots x_{l-1} \in \Sigma^l$, where each $x_i \in \Sigma$, one common analytical technique is to consider all length-k substrings $\{k_0, \ldots, k_{l-k}\}$, where $k_i = x_i\cdots x_{i+k-1} \in \Sigma^k$.
Given a hash function $h$,
we define the minimizers of the sequence as $\{w_0, \ldots, w_{l-k-w} \}$, where
\begin{equation}
w_i = \argmin_{k_j \in \{k_i, \ldots, k_{i+w-1} \} } h(k_j) .
\end{equation}


\begin{lemma}[\cite{schleimer2003winnowing}] Let $l>w$, and let $x \in \Sigma^l$ be any string of length $l$ without duplicate k-mers. Then using the notation given above, if $h$ is a min-wise independent hash function, the density of random minimizers is $\frac{2}{w+1}$. 
\end{lemma}
\begin{proof}
There are $w+1$ unique k-mers between two adjacent windows. The very first k-mer belongs only to the earlier window, and the last k-mer belongs only to the later window. When iterating across windows, a new minimizer is selected precisely when either the new (last) k-mer is the smallest (and thus the new minimizer) or when the first k-mer is the smallest (so it was the previous minimizer, and a new minimizer needs to be selected). Thus, on iterating to each new window, the probability of selecting a new minimizer is $\frac{2}{w+1}$ because of min-wise independence.
\rev{\qed}
\end{proof}


Thus, the deduplicated set of minimizers of a sequence is much fewer than $l-k-w$.
However, importantly, although we used min-wise independence above, and that property is crucial for the sister algorithm of MinHash \cite{broder1997resemblance}, we do not actually need the hash function to be min-wise independent.
This is because we simply \rev{are} using the minimizers as a sparse sampling of the k-mer space, and we do not need it to be exactly the density computed above.
Minimizers only require translation invariance, which is satisfied by any arbitrary permutation, rather than needing a random permutation.
Indeed, this fact has been exploited to construct non-uniformly-random minimizers that are more evenly distributed, or that are likely to be rarer in a genome \cite{zheng2020improved,jain2020weighted}.
These works can reduce the density from the $2/(w+1)$ of random minimizers \cite{schleimer2003winnowing} to for example $1.67/(w+1)$ using the Miniception construction scheme \cite{zheng2020improved}.

Of course, some amount of randomness is important; in the worst case, an adversarially chosen permutation that orders k-mers in the same order as appears in the sequence gives no amount of subsampling.
Or, more realistically, simply taking a lexicographic ordering on the space of k-mers works quite badly, for example, on poly-A strings where the repetitiveness of the poly-A prefix causes many distinct k-mers to be minimizers of adjacent windows, increasing density, which is bad.
One of the main empirical results of this manuscript is that for minimizer schemes based on random Gaussian convolutional filters, repetitiveness actually helps decrease density, improving rather than hurting performance.

\subsection{CNNs}
\label{ss:cnns}
CNNs are inspired by the visual cortex of mammals, and were notably demonstrated to be effective for image processing \cite{hubel1959receptive,fukushima1982neocognitron}.
CNNs are characterized by applying a set of filters in a translation-invariant noise-robust way to different parts of an image to generate a set of features, adding a pooling layer to reduce the information passed downstream, and then following up with a feed-forward neural network for analyzing the features and connecting them to a classification or prediction \cite{lawrence1997face}. \rev{Intuitively, each filter looks for a particular type of feature (e.g. an edge or corner) at every position of the image.}
Of course, this is a gross oversimplification, and modern architectures are much more complex, but this captures the basic idea.

For our analysis, we consider a simple 1D-CNN on a 4-letter alphabet $\Sigma$ with an initial convolutional layer with a single filter of size $k$ and stride-1, a max-pooling layer with patch-size $w$ and stride-1, followed up by an arbitrary feed-forward neural network---consider a multi-layer perceptron for simplicity, but it is irrelevant for our analysis---and initialized with Gaussian random weights with mean 0 and variance 1.
The results that follow generalize naturally to other finite alphabets, multiple filters, and multiple dimensions, but the analysis will be easier in this setup.

More formally, let $\mathbf{s} = [s_0, \ldots, s_{l-l}] \in \Sigma^l$ be a length-$l$ string.
Encode the string $\mathbf{s}$ in a one-hot encoding $\mathbf{e}=\mathbf{e(s)}=[e_0, \ldots, e_{4l-1}] \in \{0, 1\}^{4l}$, \rev{where $e_0=1$ iff $s[0]=$A, $e_1=1$ iff $s[0]$=C, $e_2=1$ iff $s[0]=$G, $e_3=1$ iff $s[0]=$T, $e_4=1$ iff $s[1]=$A, etc}.
Our convolutional filter $\mathbf{g}$ is initialized as a spherical Gaussian multivariate vector with i.i.d. components of length $4k$, $\mathbf{g} = [g_0, \ldots, g_{4k-1}]$.
Because our one-hot encoding expanded the length of the vector by a factor of 4, we will be using a stride of 4 instead of 1---we eschew tensor notation for the sake of simplicity.
Then, the first convolutional\footnote{We follow the common machine learning convention whereby convolution indices are not reversed; apologies for any confusion to those from outside of ML, but it makes no difference because the indices of $\mathbf{g}$ are randomly initialized.} layer is the function $L_1: \{0, 1\}^{4l} \to \mathbb{R}^{l-k}$ given by
\begin{equation}
L_1(\mathbf{e})_i = (\mathbf{e} * \mathbf{g})_{4i} = \sum_{j=1}^{4k} e_{4i+j} g_j  \textrm{ for } i \in \{0, \ldots, l-k-1 \}
\end{equation}

The second layer of of the CNN is the max-pooling function, given by $L_2: \mathbb{R}^{l-k} \to \mathbb{R}^{l-k-w}$ as
\begin{equation}
    L_2(y)_i = \max_{j \in \{0, \ldots, w-1\}} y_{i+j} \textrm{ for } i \in \{0, \ldots, l-k-w-1\}.
\end{equation}
After the max-pooling layer, the remainder of the neural network is trained to use the outputted features for the designated classification or prediction task.

It should be apparent that the setup for convolutional filters closely resembles that of the setup for minimizers, where the `hash function' is a dot product with the weights vector.
In both cases, some function is applied to k-mers in the string on a rolling basis.
Then, we keep the extreme value of the output of that function, whether it is a minimum or maximum, and pass that onto further analytical pipelines.
We address below why convolutional filters work on categorical sequences, when the usual story told is about robustness to non-categorical (ordinal/interval/ratio) noise in feature selection.

\section{Proofs}
Recall that the vector-multiply-shift family of hash functions defined by dot product with a uniform random ring element, following by division without remainder, is 2-independent.
Roughly speaking, our main claim is that the family of hash functions defined by dot product with a multivariate spherical Gaussian has the property that equally distinct k-mers are equally likely to be selected as minimizers/maximizers, and more distinct k-mers are more likely to be selected as minimizers/maximizers, but the set of minimizers/maximizers selected tends to be similar to each other.

To make this rigorous, we first need to know what we mean by the `distinctness' of a k-mer.
We define that by the total (summed) Hamming distance from all the other unique k-mers in a set; this is basically the absolute deviation of the k-mer from the set.
We will call this the `degree' of the k-mer in the set---consider a graph with k-mers as vertices and edge weights corresponding to the Hamming distance.
Our first goal is to show that the degree of a k-mer determines it's likelihood of being either the minimum or maximum under the Gaussian convolution hash family.
Later we will also show that extrema also correlated in composition, so the overall set of selected k-mers is similar to each other.

Because we are considering convolution with a multivariate spherical Gaussian, we can take advantage of a large body of literature on manipulating i.i.d. Gaussians.
However, we must first rewrite a k-mer in a one-hot encoding so that convolution with a multivariate Gaussian is well-defined (otherwise, what does it mean to multiply a nucleotide ``C(ytosine)" with a real number).
Given these preliminaries, we are now able to state our theorem.

\begin{theorem}
Consider a set $S$ of $n$ binary vectors in $\{0, 1\}^d$, all with the same number $m$ of set bits.
Define the degree $\Delta(\mathbf{x})$ of $\mathbf{x} \in S$ by
\begin{equation} 
\Delta(\mathbf{x}) = \sum_{s \in S} ||x - s||_1
\end{equation}

Let $\mathbf{g}$ be a spherical multivariate Gaussian random variable of dimension $d$ (each entry $g_i \sim \mathcal{N}(0,1)$ i.i.d.), defining a hash function $h: S \to \mathbb{R}$ by $h(\mathbf{x})=\mathbf{x} \cdot \mathbf{g}$.

Then for all $\Delta(\mathbf{x}) \ge \Delta(\mathbf{y})$
\begin{equation}
    \Pr \left(h(\mathbf{x}) = \max_{\mathbf{s} \in S} h(\mathbf{s})\right) \ge \Pr \left(h(\mathbf{y}) = \max_{\mathbf{s} \in S} h(\mathbf{s})\right) .
\end{equation}
\label{thm:math}
\end{theorem}
\begin{proof}
\vspace{-1em}
The naive approach would be to try to directly prove that $h(\mathbf{x})$ is likely to be greater than $h(\mathbf{y})$.
Unfortunately, this approach fails: because Gaussians are symmetric about 0, $\mathbb{E}(\mathbf{x}\cdot \mathbf{g}-\mathbf{y}\cdot \mathbf{g}) =(\mathbf{x}-\mathbf{y})\cdot\mathbb{E}\mathbf{g} = 0$.
Instead, in this proof we will focus on variances.

Let's order all the elements of $S$ and place $\mathbf{x}$ and $\mathbf{y}$ as the first two elements to make notation easier.
That is to say, without loss of generality, let $S = \{ \mathbf{x} = \mathbf{s}_1, \mathbf{y} = \mathbf{s}_2, \mathbf{s}_3, \ldots, \mathbf{s}_n \}$.
Let the random variable $Y_i = h(\mathbf{s}_i)$, and let $\Delta_i = \Delta(\mathbf{s}_i)$, the degree.

First, we're going to relate the degree of an element to the square of its hash's deviation from all the other hashed values.

\begin{lemma}
$\Delta_i = \mathbb{E} \sum_{j = 1}^n (Y_i - Y_j)^2 $
\label{lemma:degree}
\end{lemma}
\begin{proof}
Each $Y_i$ is a sum of 1D Gaussians from $\mathbf{g}$, so $Y_i - Y_j$ precisely cancels out any shared terms, leaving us with positive copies of the Gaussians only in $Y_i$, and negative copies of the Gaussians only in $Y_j$.
Because all our binary vectors had the same number of set bits, the total number of both positive and negative Gaussians that are summed to form $Y_i - Y_j$ is precisely the Hamming distance $||\mathbf{s}_i - \mathbf{s}_j||_1$.
Since we canceled out any shared terms, all the remaining Gaussians are i.i.d., and due to the symmetry of a Gaussian under reflection, that means that we have that many independent Gaussians.
The sum of $t$ independent $\mathcal{N}(0,1)$ Gaussians is just another Gaussian with distribution $\mathcal{N}(0,t)$ (i.e. a variance-$t$ Gaussian), and the second moment of a Gaussian is just its variance.
Thus, $\mathbb{E} (Y_i - Y_j)^2 = ||\mathbf{s}_i - \mathbf{s}_j||_1$, and the lemma follows by linearity of expectation.
\rev{\qed}\end{proof}

Interpreting Lemma \ref{lemma:degree}, the degree $\Delta_i$ gives the expected squared deviation after the dot product.
Arguing solely from expectations, we expect the lowest degree items to have the lowest squared deviation after the dot product, and therefore be closest to the mean.
Conversely, high degrees correspond to being further away from the mean, and the set member that has the highest squared deviation from the mean has to be either the min or the max.
The joint distribution is symmetric about the origin, so the probability of being the min or the max are equal.
Thus, we need only show that with probability at least 0.5, $Y_1$ has as high of a squared deviation as $Y_2$.

We now define a new random variable by the difference of the squared deviations
\begin{equation}
Z = \sum_{ j =1}^n (Y_1 - Y_j)^2 - \sum_{ j =1}^n  (Y_2 - Y_j)^2 .
\end{equation}
When $Z\ge 0$, $Y_1$ has at least as high of squared deviation as $Y_2$, and vice versa when $Z \le 0$.
Of course, $\mathbb{E}Z = \Delta_1 - \Delta_2 \ge 0$, but expectation is insufficient for showing that with probability 0.5, $Z \ge 0$.
That is thus the goal for the remainder of this proof.

As in the proof of Lemma \ref{lemma:degree}, we use the fact that the $Y_i$'s arise as sums of selected Gaussians from $\mathbf{g} = \{g_1, \ldots, g_d\}$.
Let $t = ||\mathbf{s}_i - \mathbf{s}_j||_1$, which is even because all the bit vectors have the same number of set bits.
Then $(Y_i - Y_j)$ is the sum of $t/2$ entries from $\mathbf{g}$ minus the sum of another \textit{distinct} $t/2$ entries from $\mathbf{g}$, which we can write as 
\begin{equation}
    (Y_i - Y_j) = g_{\alpha_1} + \cdots + g_{\alpha_{t/2}} - g_{\alpha_{t/2+1}} - \cdots - g_{\alpha_{t}}
\end{equation}
Then $(Y_i - Y_j)^2$ is the sum of $t$ squared Gaussians and all of the cross terms, which are either positive or negative copies of products of two independent Gaussians (for notational simplicity, we did not bother writing out which copies are positive or negative):
\begin{equation}
    (Y_i - Y_j)^2 = \sum_{u=1}^{t} g_{\alpha_u}^2 + \sum_{u \ne v} \pm g_{\alpha_u} g_{\alpha_v}
\end{equation}

Summing everything together, we can rewrite
\begin{align}
    Z &= \sum_{i=1}^d c_i g_i^2 + \sum_{i=1}^d \sum_{j=i+1}^d c_{i, j} g_i g_j ,
    \label{eq:u-int}
\end{align}
where the $c_i$'s and $c_{i,j}$'s are unknown positive or negative integer coefficients.
Note that we get nontrivial coefficients because although within $(Y_i - Y_j)$ all the component Gaussians are unique, this is no longer true when summing everything together.

The signs of the $c_{i,j}$'s do not matter.
For i.i.d. Gaussians $g_i$ and $g_j$, we can rewrite 
\begin{equation}
    g_i g_j = \frac{1}{4} (g_i + g_j)^2 - \frac{1}{4} (g_i - g_j)^2,
\end{equation} which is just the difference of two independent parameter-1 $\chi^2$ variables with the same distribution.
The difference of two mutually independent random variables of the same distribution is always symmetric about the origin---in this case, the difference of two independent $\chi^2$ variables gives a variance-gamma distribution \cite{klar2015note}.
Thus, the right term $\sum_{i=1}^d \sum_{j=i+1}^d c_{i, j} g_i g_j$ of equation \ref{eq:u-int} is symmetric about the origin.

Additionally, using $\mathbb{E} g_i^2 = 1$ and $\mathbb{E} g_i g_j = 0$, that implies that $\mathbb{E}\sum_{i=1}^d c_i g_i^2 = \mathbb{E}Z = \Delta_1 - \Delta_2 \ge 0$.
The summation $\sum_{i=1}^d c_i g_i^2$ can be represented as a sum of a bunch of positive and negative copies of parameter-1 $\chi^2$ variables of the same distribution---we have $c_i$ copies $g_i^2$, which is a parameter-1 $\chi^2$ variable---with at least as many positive as negative copies.
Because each negative $\chi^2$ variable is pairwise independent of every positive variable, we can pair each negative copy with an independent positive copy of the same distribution.
Using the same logic as above, those pairs when summed are symmetric about the origin (and again variance-gamma distributions).

The only asymmetric part of $Z$ is therefore the leftover $\chi^2$ terms (if any), which are strictly positive.
Because $Z$ is the sum of something symmetric about the origin and a strictly positive component, $\Pr(Z \ge 0) \ge 0.5$, concluding the proof.

\rev{\qed}\end{proof}

\begin{theorem}
Consider the universe $U$ of all $n$ binary vectors in $\{0, 1\}^d$ with $m$ set bits.
Let $\mathbf{g}$ be a spherical multivariate Gaussian random variable of dimension $d$ (each entry $g_i \sim \mathcal{N}(0,1)$ i.i.d.), defining a hash function $h: U \to \mathbb{R}$ by $h(\mathbf{x})=\mathbf{x} \cdot \mathbf{g}$. Let $\hat{\mathbf{s}}$ be a (universal) maximizer of $U$ under $h$, defined by $h(\hat{\mathbf{s}}) = \max_{\mathbf{s} \in S} h(\mathbf{s}) \equiv M$.
Then for any $\mathbf{x}, \mathbf{y} \in U$ such that $||\mathbf{x}-\hat{\mathbf{s}}||_1 < ||\mathbf{y} - \hat{\mathbf{s}}||_1$,
    $\mathbb{E} h(\mathbf{x}) > \mathbb{E} h(\mathbf{y})$, the conditional expectations given that $\hat{\mathbf{s}}$ is a maximizer.
\label{thm:math2}
\end{theorem}
\begin{proof}
The proof is elementary; intuition is being correlated with a maximizer makes you larger on average.

First note that the event that there are two distinct maximizers has measure zero, so in this proof we assume $\hat{\mathbf{s}}$ to be unique.
The theorem follows as a straight-forward consequence of linearity of expectation, after conditioning on the maximizer $\hat{s}$---for ease of notation, all expectations in this section are conditional expectations.
Consider the conditional expectations of the entries $g_i$.
Let's abuse notation and treat bit vectors as sets of indicies; for example, to say that $i \in \mathbf{x}$ if $\mathbf{x}_i=1$. Then we know two facts:
\begin{enumerate}
    \item $\mathbb{E} (g_i | i \in \hat{\mathbf{s}}) = \frac{M}{m}$, because we have no information to break the symmetry among set bits of $\hat{\mathbf{s}}$.
    \item $\mathbb{E} (g_i | i \not \in \hat{\mathbf{s}}) < \frac{M}{m}$, because otherwise the unique maximizer would not be $\hat{\mathbf{s}}$.
\end{enumerate}
Then we can explicitly compute $\mathbb{E}h(\mathbf{x})$  (or respectively $\mathbb{E}h(\mathbf{y})$) as follows:
\begin{align*}
    \mathbb{E}h(\mathbf{x}) &= \mathbb{E} \left[ \sum_{i \in \mathbf{x}} g_i \right] = \mathbb{E} \left[ \sum_{i \in \mathbf{x} \cap \hat{\mathbf{s}}} g_i + \sum_{i \in \mathbf{x} \smallsetminus \hat{\mathbf{s}}} g_i \right] = \left( n - \frac{1}{2}||\mathbf{x} - \hat{\mathbf{s}}||_1 \right) \frac{M}{m} + \frac{1}{2}||\mathbf{x} - \hat{\mathbf{s}}||_1 \mathbb{E} (g_i | i \not \in \hat{\mathbf{s}})
\end{align*}
Then subtracting and substituting in the distance assumption, we get
\begin{align*}
\mathbb{E}h(\mathbf{x}) - \mathbb{E}h(\mathbf{y}) &= \left(  \frac{1}{2}||\mathbf{y} - \hat{\mathbf{s}}||_1 - \frac{1}{2}||\mathbf{x} - \hat{\mathbf{s}}||_1 \right) \frac{M}{m} + \frac{1}{2}\left( ||\mathbf{x} - \hat{\mathbf{s}}||_1 - ||\mathbf{y} - \hat{\mathbf{s}}||_1 \right) \mathbb{E} (g_i | i \not \in \hat{\mathbf{s}}) > 0 ,
\end{align*}
which completes the proof.
\rev{\qed}\end{proof}

We can now directly apply Theorems \ref{thm:math} and \ref{thm:math2} to get Theorem \ref{thm:main}, which we state in full generality.

\begin{theorem}
\rev{Consider a CNN acting on one-hot encoded categorical vectors with a single convolutional filter following by max-pooling, both of stride-1. Then the output of the max-pooling layer is precisely equivalent to choosing a minimizer over the windows/patches that are max-pooled. If the weights of the filter are initialized as a spherical Gaussian multivariate, then the implied minimizer hash function is mostly random but has the following properties:
\begin{enumerate}
    \item It is at least as likely to choose the the most distinct categorical feature in a window as any other k-mer, where distinctness is measured in absolute Hamming deviation from the other features;
    \item And the set of minimizers it chooses overall are likely to be closer in Hamming distance to each other than a random set of k-mers.
\end{enumerate}}
\label{thm:main}
\end{theorem}
\begin{proof}
A one-hot encoding of a categorical vector precisely gives binary vectors with the same number of set bits.
The setup follows directly from the definitions set up in Sections and \ref{ss:minimizers} and \ref{ss:cnns} and illustrated in Figure \ref{fig:cnn}: convolution with a randomly initialized spherical Gaussian is clearly \textit{some} kind of hash function, vector-mapped across positions in the sequence (though perhaps not one with the standard properties), and a max-pooling operation is equivalent to taking a minimum across windows.
However, we proved in Theorem \ref{thm:math} that a dot-product hash function with spherical Gaussian parameters preferentially selects more distinct k-mers, but is equally likely to select k-mers that are equally distinct.
The hash function so defined striates k-mer features by degree (Hamming distance from all other k-mers), and all k-mers of a given degree have the same probability of being chosen as a minimum or maximum, but at least as high a probability as any k-mer of lower degree, proving the first part of Theorem \ref{thm:main}.
The second part of Theorem \ref{thm:main} is a consequence of the Theorem \ref{thm:math2}, as we showed that the closer a k-mer is to the universal minimizer, the more likely it is to be selected as a local minimizer.
\rev{\qed}
\end{proof}

\section{Experimental Results}
\begin{figure}[t]
    \centering
	\includegraphics[height=0.33\columnwidth]{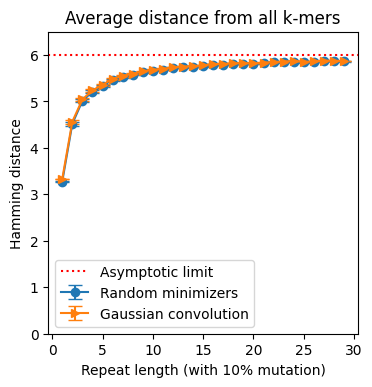}
	\hfill
	\includegraphics[height=0.33\columnwidth]{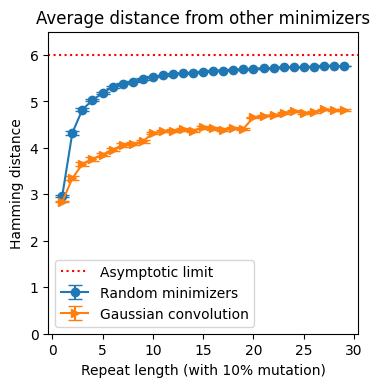}
	\hfill
	\includegraphics[height=0.33\columnwidth]{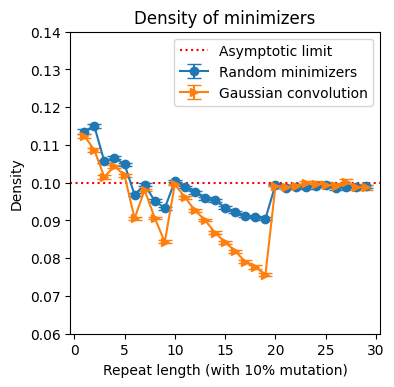}
		\vspace{-1.5em}
		\caption{We generated repetitive sequences of length 1007, with repeat length on the x-axis and a 10\% repeat mutation substitution rate. We used a k-mer size of 8 and a window size of 19, for an expected random minimizer density of 0.1 over the 1000 k-mers of the sequence. All data points are an average of 400 random runs, with standard error bars. \textbf{(left)} The average Hamming distance of the minimizers from all k-mers in the sequence is basically the same for random minimizers and Gaussian convolutions; this is consistent with Theorem \ref{thm:math}. \textbf{(center)} However, the average distance from the other minimizers is significantly lower for Gaussian convolutions, as predicted by Theorem \ref{thm:math2}. \textbf{(right)} On repetitive windows of a sequence, Gaussian convolutions tend to have lower density than random minimizers.}	
    \label{fig:exp}
    \vspace{-1.5em}
\end{figure}

While we have proven general properties of using Gaussian convolution as a minimizer selection scheme, our theorems do not say anything about in what situations these behaviors manifest. To address that, we turn to numerical simulation (Figure \ref{fig:exp}).
The correct interpretation of our theory is that k-mers which are further away (in Hamming distance) from other k-mers in a window, but similar to other minimizers, are more likely to be selected.
This phenomenon thus primarily occurs in repetitive sequences, where k-mers within a window are similar to each other.
We thus run our experiments on repetitive sequence with mutations, both by simulating tandem repeats with mutations, and measuring on real human telomeres.\footnote{Code for Figures \ref{fig:exp} and \ref{fig:telomere} are available at: \url{https://github.com/yunwilliamyu/gaussian-minimizer-density}}
Finally, we finish by designing a CNN architecture to learn a neural embedding of simulated short-reads that respects linear distance along a genome, the first step in building a genome assmebly graph.\footnote{Code and models for Figure \ref{euclidean-embeddings} available at: \href{https://storage.googleapis.com/progress-towards-assembly/colab.html}{https://storage.googleapis.com/ progress-towards-assembly/colab.html}}.

\subsection{Simulated tandem repeats}
We chose a k-mer size of 8 for our experiments, and a total sequence length of 1007 (so that we have 1000 k-mers).
We chose a window size of $w=19$, as the density of random minimizers is expected to be $2/(w+1) = 0.1$.
We then swept over repeat lengths from 1 to 30, with a 10\% mutation rate.
For example, if we chose a repeat length of 5, the simulation begins by randomly choosing 5 letters for the repeat, e.g. ``ACCCT''; then, this 5-letter sequence is repeated to fill up a sequence of length 1007, but with 10\% of the nucleotides randomly substituted.
This is thus a simplified model of a long tandem repeat with mutations.

We measured three different quantities: (1) average distance from all k-mers, (2) average distance from other minimizers, and (3) the density of the minimizers. The first two quantities correspond to empirical validation of our two technical theorems.
Random minimizers and Gaussian convolutions are on average the same distance from all k-mers, asymptotically approaching 6 because 8-mers with a 4-letter alphabet on average share 2 characters; importantly, Theorem \ref{thm:math} says that Gaussian convolutions should be on average far from other k-mers, and this turns out to match random minimizers.
However, Gaussian convolutions minimizers are on average much closer to each other than random minimizers are; this is due to Theorem \ref{thm:math2}.

One seeming consequence is that Gaussian convolution minimizers have lower density in repetitive windows than random minimizers do. Once the repeat length grows to the window size, the density effects disappear entirely, but we see a minimum density as low as $0.0756 = 1.512/(w+1)$ when the repeat length is one less than the window size.
It is true that random minimizers show a similar lowered density in repetitive windows, but the effect is much more pronounced for Gaussian convolution minimizers.
As a caveat, the decreased density is not alone a reason to choose Gaussian convolution minimizers though, as Gaussian convolution minimizer density only decreases for repetitive regions, whereas existing work such as Miniception allows a consistent lowered density of around $1.67/(w+1)$ over both repetitive and non-repetitive regions \cite{zheng2020improved}.
Instead, these experiments are meant simply to partially explain the performance of convolutional filters.

\subsection{Real human Telomeres}
We extracted the annotated telomeric regions from version 2.0 of the Telomere to Telomere consortion CHM13 project \cite{rhie2022complete}. Each chromosome has at least two annotated telomeric regions at the ends, and some chromosomes have an additional telomeric region; we labelled all telomeric regions in order of position on the genome (e.g. the first telomere on chromosome 3 is telomere 3a, the second is 3b, etc.). Then, using the same parameters (k-mer size of 8 and window size of 19, for an expected random minimizer density of 0.1 over the telomere), we compared the density of using random minimizers vs. Gaussian convolutions. These are highly repetitive regions of the genome.

We see in figure \ref{fig:telomere} that the performance on real repetitive data generally matches the simulated data: the Gaussian convolutions in expectation have lower density than random minimizers---occasionally the trend is the opposite direction, but we averaged over 400 random runs.
This behavior is very different from that of other k-mer selection schemes, like open sync-mers \cite{edgar2021syncmers}, which can sometimes exhibit much higher density in repetitive regions.

\begin{figure}[t]
    \centering
    \includegraphics[width=1\columnwidth]{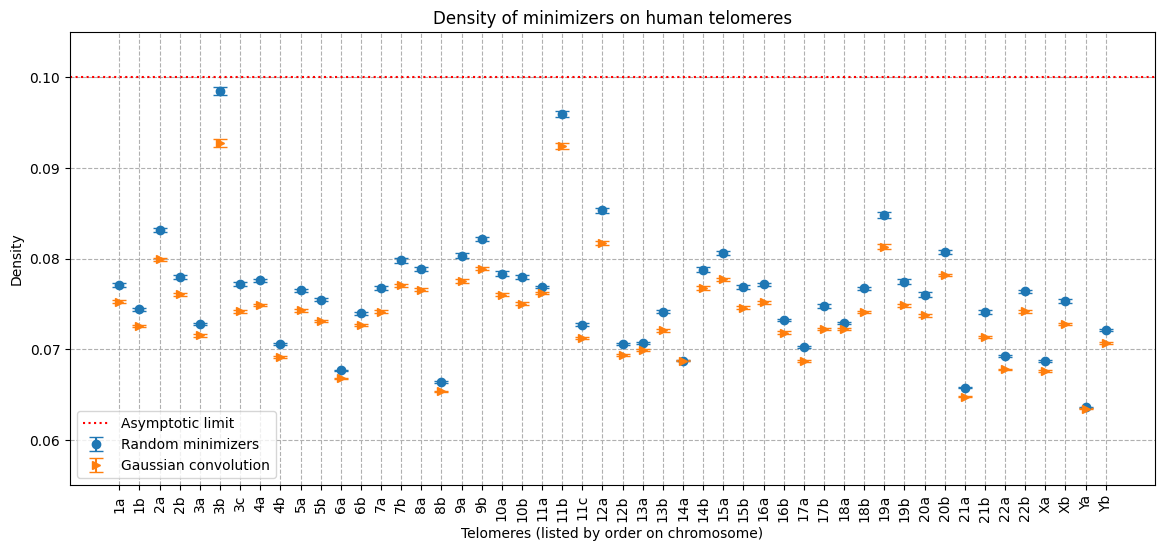}
    \caption{Density on human telomeres, using k-mer size of 8 and window size of 19, for an expected random minimizer density of 0.1. All data points are an average of 400 random runs, with standard error bars. The general trend on real data matches the simulated data: Gaussian convolutions have lower density than random minimizers. Note that the x-axis is annotated with the chromosome and telomere number.}
    \label{fig:telomere}
\end{figure}

\subsection{Genome assembly graph of SARS-CoV-2}
\begin{figure*}[tp]
\begin{center}
\centerline{\includegraphics[width=0.35\textwidth]{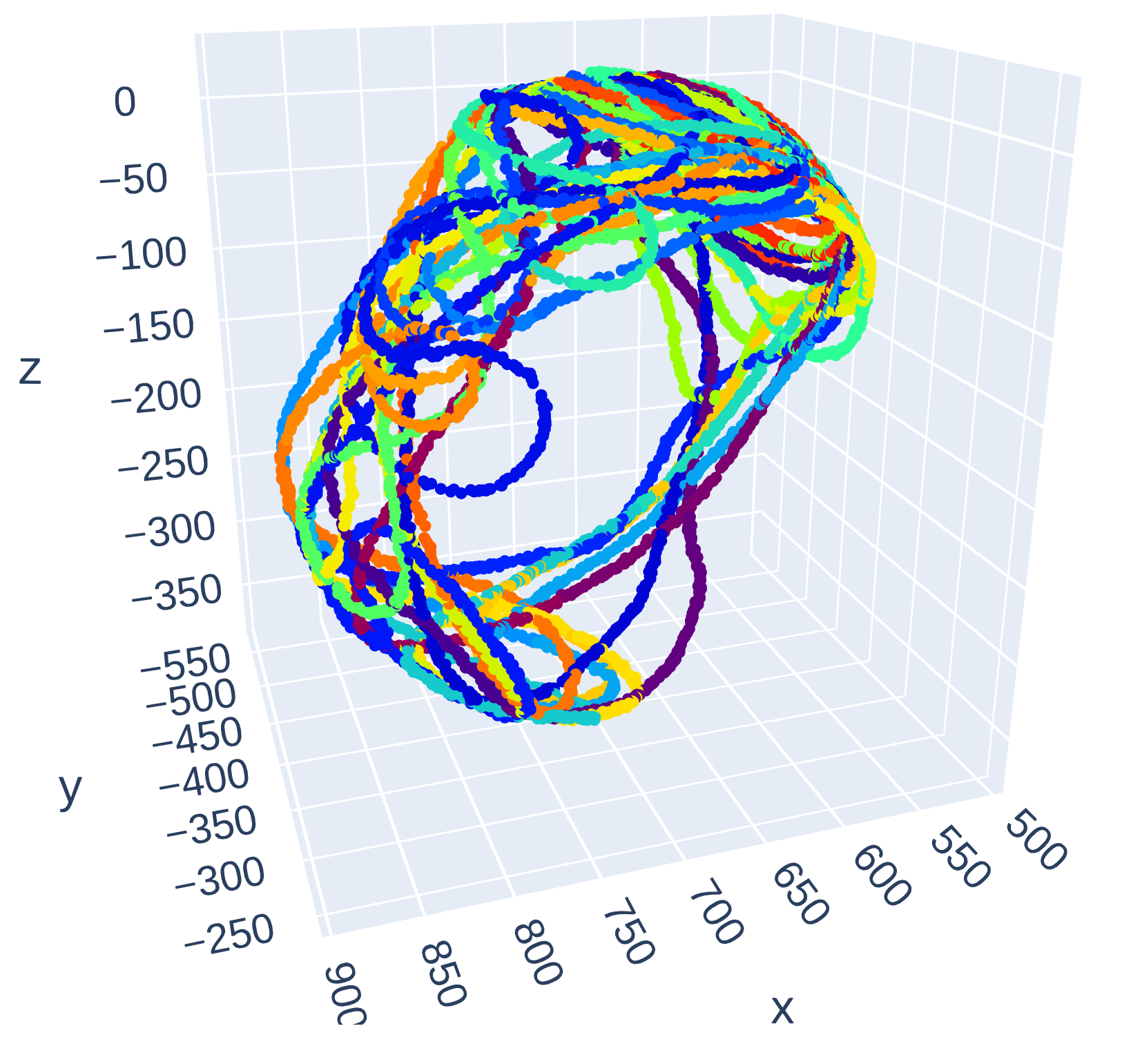} \includegraphics[width=0.6\textwidth]{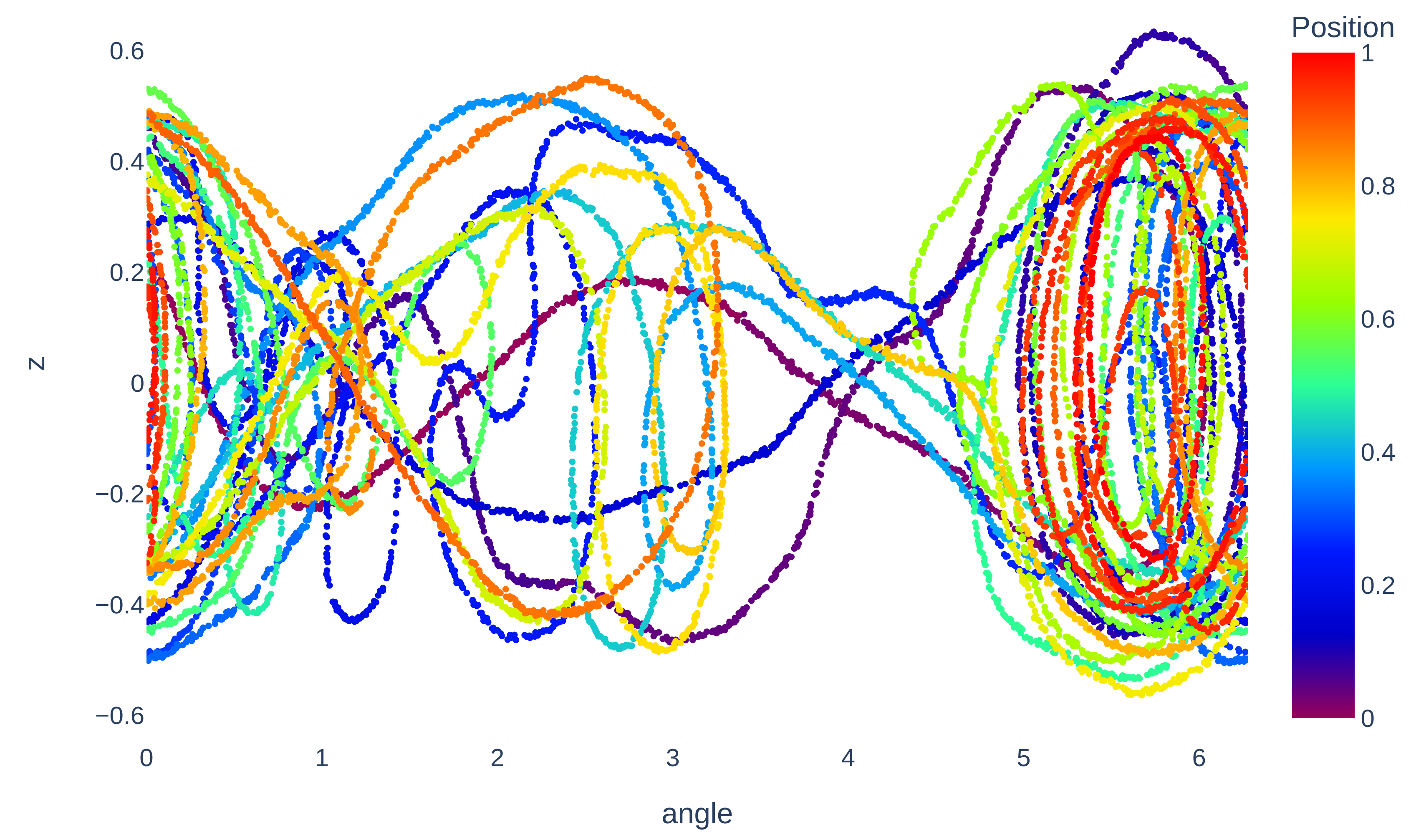}
}
\centerline{\includegraphics[width=0.5\textwidth]{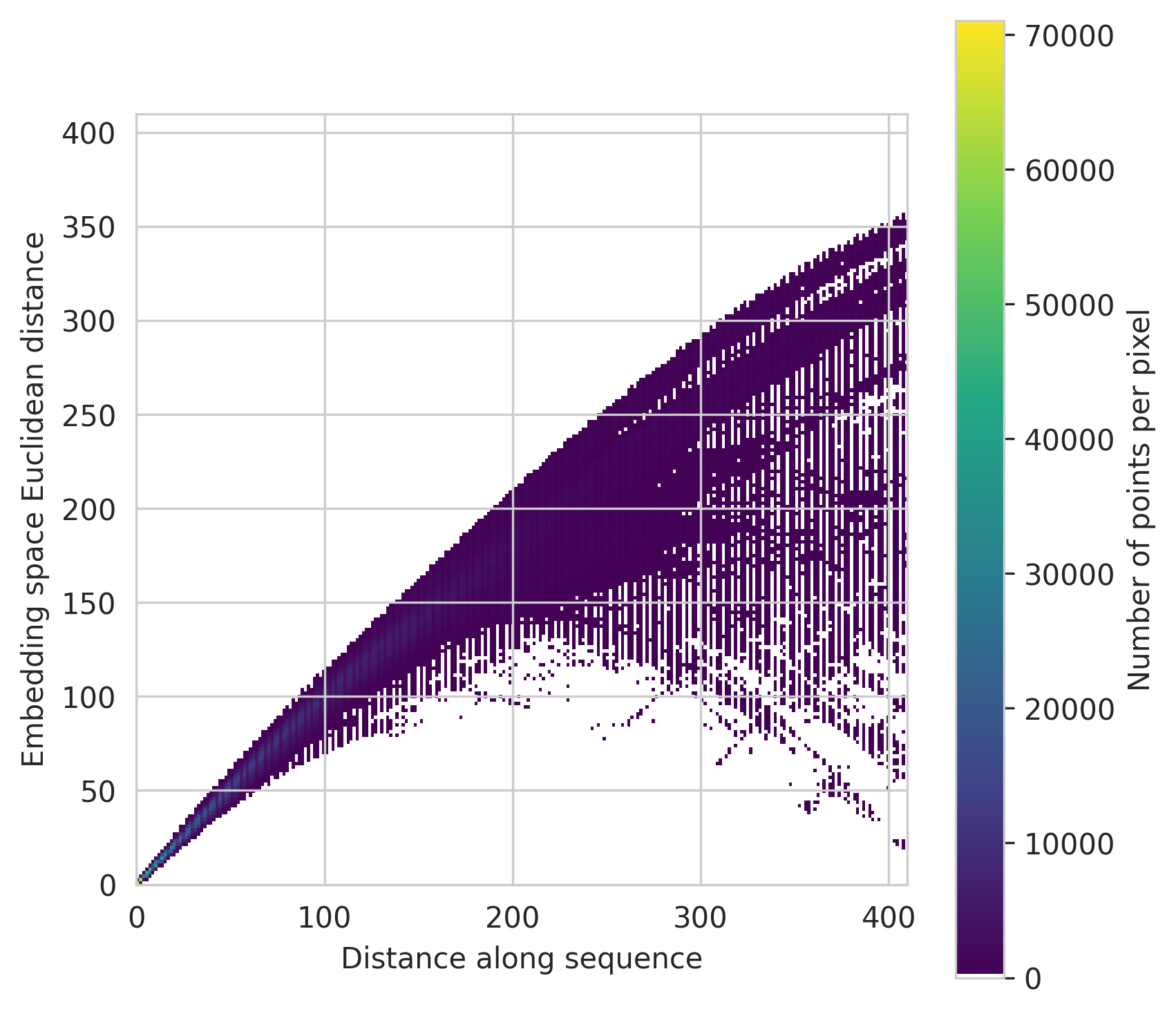} \includegraphics[width=0.5\textwidth]{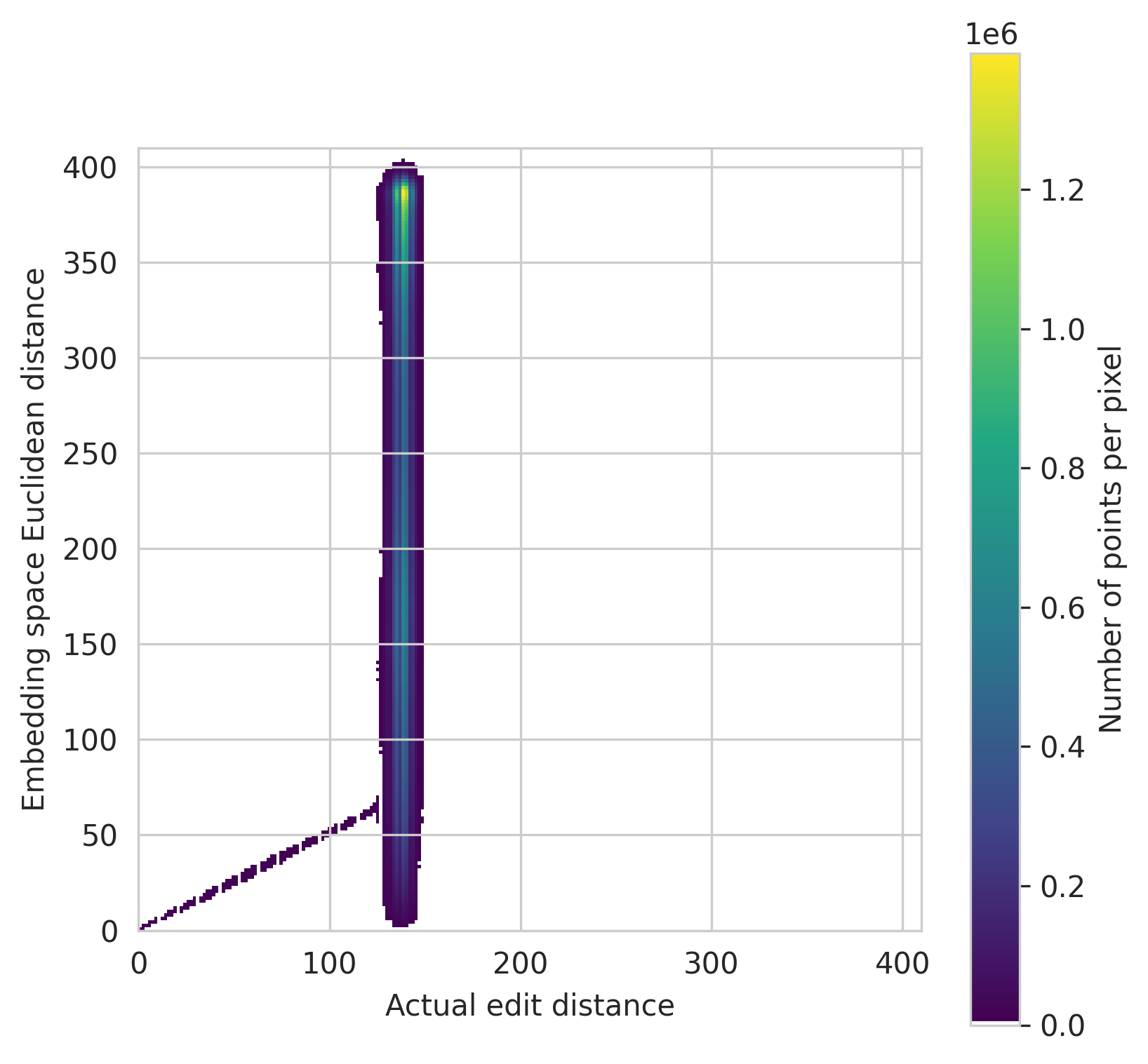}
}
\caption{Euclidean embedding of synthetic short-reads (length $l=256$) from the SARS-CoV-2 genome as a pseudo-read-overlap assembly graph. \\(\textbf{top left}) Actual embedding was into 3D, attempting to preserve local distances along the sequence. \\(\textbf{top right}) For ease of visualization, we further transformed the data onto a plane to better show the trajectories. Note that the embedding was not given any of the positional (\textbf{color}) information, but instead managed to recover that information implicitly from just the raw reads. Note that this is only the first step to building an assembly graph, and that it was substantially slower than existing software, but is only here as a visual proof of concept. \\(\textbf{bottom left}) Heat map comparing embedding distance with distance along sequence in low distance regime. The embedding space distance generally increases with distance along sequence, until the reads are too far apart along the sequence; note that the reads were only trained with a maximum distance of $l=256$, so it is not surprising that distant reads have closer embedded distances. \\(\textbf{bottom right}) Actual Levenshtein edit distances between reads of length $l=256$ obviously cannot exceed 256, and in practice for synthetic reads from the SARS-CoV-2 genome do not exceed 150. Note that edit distance corresponds roughly to twice the embedding distance in the small-distance regime; this is because an overlap distance of 1 requires two edits, an insertion and a deletion. Although it does occasionally happen that reads far apart in edit distance are close on the embedding, this is relatively rare. More importantly, if the distance along sequence or edit distance are small, then the the embedding distance is also small.}
\label{euclidean-embeddings}
\end{center}
\vspace{-1em}
\end{figure*}

We started with the SARS-CoV-2 genome \texttt{NC\_045512.2}, which is just under 30 kilobases long.
A corpus of synthetic short-reads of length 256bp were generated to a depth of coverage of 200.
The synthetic reads had an error rate of 1\%, evenly divided three-ways among insertions, deletions, and substitutions.
Reads with deletions were padded randomly to ensure constant length; insertions were uniformly chosen among the four nucleotides, and substitutions were uniformly chosen among the other three nucleotides.

Our architecture was a 1D-CNN designed to pick out k-mers for $k \in \{1, \ldots, 10\}$, with a total of 8218 total convolutional filters, with increasing numbers of filter for larger $k$. With windows and strides proportional to $k$, this resulted in 253824 neurons after max-pooling, followed up with eight fully-connected layers (size 512, 256, 128, 64, 32, 16, 8, and then 3) with GELU activations between fully-connected layers to reduce the dimensionality down to 3.
The total number of trainable parameters of the network was 130,437,175. We chose an embedding dimension of 3 for two reasons: (1) to visually inspect the results, and (2) because assembly graphs are known to be non-planar.

Then, for training, in each epoch, we randomly generate one `adjacent' read (a rotation followed by random substitutions) for each read in our corpus, and train the network to produce an embedding that recapitulates the rotation length, which we use as a proxy for training the network on edit distance, which takes quadratic time to compute.
Thus, the network learns embeddings for pairs of reads similar to our reads, without the danger of overfitting on any particular pairs.
For this experiment, we trained for 16000 epochs, though the training and validation loss both started plateauing at around 3000 epochs.

Figure \ref{euclidean-embeddings} shows that the neural network did indeed learn an embedding that respected the linear distance along the genome\rev{---i.e. reads that are neighbors in the original genome are still neighbors in the embedding}.
We should add though that this training took 27 hours on Google Colab with an A100 GPU to just get an embedding, without even a full assembly, as compared against minutes to do the same with a modern assembler, so this is highly impractical.
However, it does show that neural networks are able in principle to \textit{de novo} generate an almost assembly graph, \rev{by simply connecting each read with its nearby neighbors in the embedding. Additionally, note that for visualization purposes, we embedded into 2D/3D, which imposes a heavier information bottleneck, and that a practical deep learning assembler would likely use larger embedding dimension.}



\section{Discussion}

Minimizers and CNN features do not depend on more classical notions of uniform hashing.
Indeed, this idea has been independently explored for both minimizers and CNNs.
On the minimizer front, there are modifications to the hash function for better density \cite{zheng2020improved} or postprocessing the permutation using inverse document/genome frequency \cite{jain2020weighted}.
For CNNs, feature construction is largely trained, instead of designed, but there is also work specifically on architectures for promoting better features \cite{caron2018deep}.
Here we show that just initializing with Gaussian weights causes some correlation to distinctiveness, a proxy for inverse frequency related to density in repetitive regions.
Notably, clusters of similar data are mapped close together with a Gaussian convolution, which does not happen for random minimizers.
This is a known property of CNN filters when it comes to ordinal data, and we show it remains true for categorical data.
In biological terms, all k-mers matched by a spaced seed get mapped close to each other by a Gaussian convolution. 

There is recent work in the computational biology literature that has demonstrated it is possible to use deep neural networks to train better minimizer schemes \cite{hoang2022deepminimizer}, but that work treats minimizers as just another function to be learned.
This does highlight the primary limitation of our theory, which is that our analyses only hold at initialization time, and more sophisticated mathematical machinery is necessary to understand behavior after training.
We did follow up with an empirical exploration of applying and training CNNs for neural embeddings of reads for assembly graphs, but that was in many ways beyond the scope of what we actually proved.

Some of the other seeming limitations are related to generalizing our results beyond a single layer with a single convolutional filter and stride lengths of one, but these limitations are actually easily overcome.
Our theorems generalize to 2D (or 3D) filters, at only the cost of notational complexity.
Longer stride lengths simply amount to a sparsification of the redundancy found in neighboring minimizers; minimizer schemes sparsify through deduplication of the minimums, but that operation is not easily vectorized, so CNNs instead sparsify by using longer stride lengths.
Stride-length sparsification is less efficient space-wise than a full deduplication, but with high probability conveys the same information, provided that the strides are not too long (Figure \ref{fig:cnn}).
Also, multiple filters within a layer correspond to just taking multiple minimizers within each region;
Another reason why this may be helpful for CNNs is due to the fact that memorizing all the output values of a single neuron is possible for CNNs, but rather hard at training time; thus, breaking up the minimizer information across multiple filters likely makes it easier to train.
Furthermore, having multiple layers of convolutional filters effectively creates minimizer schemes that are able to learn more complicated hash functions than just Gaussian dot products.
These theoretical connections deserve further study.

In this manuscript, through a careful probability computation, we proved that the family of hash functions defined by Gaussian dot products is useful for minimizer style analysis.
By recasting CNN convolutional filters in a hash-function based framework, we were able to demonstrate their equivalence to minimizers, one of the workhorses of computational biology.
Furthermore, we validated our theory empirically, showing a lowered density in repetitive regions for Gaussian convolution minimizers, both in simulation, and on real human telomeres.
Lastly, we demonstrated that convolutional neural networks can learn neural embeddings for the construction of assembly graphs, though at present, it is still so slow as to be impractical.
We hope that the connection proves fruitful for both methods, enabling the design of better minimizers, as well as providing some mathematically rigorous justification for why CNNs work in categorical data space, and hinting at future hybrid algorithms taking advantage of both classical algorithms and deep learning.

\section{Acknowledgements}
I thank Daphne Ippolito, Jim Shaw, and David Rolnick for fruitful discussions, and acknowledge the support of the Natural Sciences and Engineering Research Council of Canada (NSERC), (grant RGPIN-2022-03074).

%
%
%
\bibliographystyle{splncs04}
\bibliography{mybib}

\end{document}